\documentclass[fullpaper,9pt]{IEEEtran}

\usepackage{amsmath,amssymb,enumerate,subfigure,multirow,hhline,color}
\usepackage{xcolor}
\usepackage{hyperref}
\usepackage{graphicx}
\usepackage{graphics}
\usepackage[sort]{cite}
\usepackage{amsthm}
\usepackage{algorithm}
\usepackage{algpseudocode}

\usepackage{tcolorbox}
\usepackage{epsfig,psfrag}
\usepackage{tabularx}
\usepackage{tabulary}

\usepackage[sort]{cite}

\usepackage{hyperref}
\hypersetup{breaklinks=true,colorlinks=false,linkcolor=blue,citecolor=blue}

\usepackage[noabbrev,capitalise,nameinlink]{cleveref}

\DeclareMathOperator{\argmin}{arg\,min}

\newtheorem{theorem}{Theorem}

\newtheorem{assumption}{Assumption}
\newtheorem{remark}{Remark}

\newtheorem{problem}{Problem}
\newtheorem{lemma}{Lemma}
\newtheorem{definition}{Definition}
\newtheorem{proposition}{Proposition}

\allowdisplaybreaks

\usepackage{amsfonts}
\usepackage{textcomp}

\title{New Versions of Gradient Temporal Difference Learning}

\author{Donghwan Lee, Han-Dong Lim, Jihoon Park, and Okyong Choi
\thanks{$^*$This material is based upon work supported by the National Research Foundation under Grant NRF-2021R1F1A1061613. The work was supported by the BK21 FOUR from the Ministry of Education (Republic of Korea) (Corresponding author: Donghwan Lee). This work was supported by Institute of Information communications Technology Planning Evaluation (IITP) grant funded by the Korea government (MSIT)(No.2022-0-00469)}
\thanks{D. Lee, H. Lim, and J. Park are with the Department of Electrical Engineering,
KAIST, Daejeon, 34141, South Korea {\tt\small
donghwan@kaist.ac.kr}. O. Choi is with Shinhan bank research team, Seoul, South Korea.}
}

\begin{document}

\maketitle

\begin{abstract}
Sutton, Szepesv\'{a}ri and Maei introduced the first gradient temporal-difference (GTD) learning algorithms compatible with both linear function approximation and off-policy training. The goal of this paper is (a) to propose some variants of GTDs with extensive comparative analysis and (b) to establish new theoretical analysis frameworks for the GTDs. These variants are based on convex-concave saddle-point interpretations of GTDs, which effectively unify all the GTDs into a single framework, and provide simple stability analysis based on recent results on primal-dual gradient dynamics. Finally, numerical comparative analysis is given to evaluate the new approaches.
\end{abstract}

\begin{IEEEkeywords}
Reinforcement learning (RL), temporal-difference (TD) learning, optimization, saddle-point problem, convergence, stability
\end{IEEEkeywords}

\section{Introduction}

Temporal-difference (TD) learning~\cite{sutton1988learning} is one of the most popular reinforcement learning (RL) algorithms~\cite{sutton1998reinforcement} for policy evaluation problems. However, its main limitation lies in its inability to accommodate both off-policy learning and linear function approximation for convergence guarantees, which has been an important open problem for decades. In 2009, Sutton, Szepesv\'{a}ri, and Maei~\cite{sutton2009convergent,sutton2009fast} introduced the first TD learning algorithms compatible with both linear function approximation and off-policy training based on gradient estimations, which are thus called gradient temporal-difference learning (GTD).

The goal of this paper is to propose new variants of GTDs, and new analysis template for convergence analysis. The main pathways to these developments are based on convex-concave saddle-point interpretations of GTDs, which were first introduced in~\cite{macua2015distributed} based on the Lagrangian duality and in~\cite{dai2018sbeed} based on the Fenchel duality~\cite{Boyd2004}. In particular, GTD2, proposed in~\cite{sutton2009fast}, can be interpreted as a stochastic primal-dual gradient dynamics (PDGD) of a convex-concave saddle-point problem, and hence, its convergence analysis can be approached from a different angle using optimization theory~\cite{macua2015distributed,dai2018sbeed}.
These interpretations were subsequently applied to distributed RL problems in~\cite{cassano2019distributed,ding2019fast,wai2018multi,lee2022distributed}.

Although the saddle-point perspectives can provide unified viewpoints and greater flexibilities in analysis \& design of GTDs and RLs, to the authors' knowledge, their potentials have not been fully investigated yet. Motivated by this insight, we develop new versions of GTD which are unified with GTD2~\cite{sutton2009fast} in a single framework through saddle-point formulations. The main contributions of this paper are summarized as follows:
\begin{enumerate}
\item New algorithms: Three new versions of GTDs are proposed, which are named GTD3, GTD4, and GTD5. These variants, especially GTD4 and GTD5, can be viewed as regularized GTD2 algorithms, where the regularization potentially improves the convergence empirically. From simulation experiments, their convergence and performance are evaluated.

\item Unified saddle-point perspectives: The proposed versions can be interpreted in a unified way based on saddle-point interpretations, which are derived from slightly different angles from those in~\cite{macua2015distributed,dai2018sbeed}.

\item Comparative analysis: Comprehensive numerical experiments are given to compare convergence of the proposed GTDs and GTD2 in~\cite{sutton2009fast}. Empirically, it turns out that the proposed GTD4 and GTD5 tend to converge faster than the other methods for the randomly generated 5000 environments.

\item General analysis templates: In existing GTD algorithms, the convergence analysis mainly exploits the ODE (ordinary differential equation) model-based stochastic approximation theory~\cite{kushner2003stochastic}, where the main challenge is proving the asymptotic stability of the ODE model corresponding to the underlying algorithm. This approach does not allow general and formal analysis frameworks because the asymptotic stability of the ODE model significantly depends on the specific algorithm, and it is in general hard to establish the stability of the ODE model. On the other hand, the proposed analysis applies the recent asymptotic stability theory of primal-dual gradient dynamics (PDGD)~\cite{qu2018exponential}, where control theoretic frameworks for stability analysis of PDGD are developed. Using this recent result, we provide a new template for convergence analysis of RL algorithms based on the saddle-point formulations. This framework leads to simple and unified convergence analysis for linear RL algorithms based on saddle-point formulations. Moreover, the new template allows us to easily analyze the proposed new versions of GTD algorithms derived based on the saddle-point formulations. This template can potentially be applied to many other RL variants in the future.

\end{enumerate}

Related previous works are briefly summarized as follows: As mentioned before, saddle-point perspectives of GTDs and RLs were introduced in~\cite{macua2015distributed,dai2018sbeed} based on the Lagrangian duality~\cite{macua2015distributed} and Fenchel duality~\cite{dai2018sbeed}. These ideas were applied to distributed RL problems in~\cite{cassano2019distributed,ding2019fast,wai2018multi,lee2022distributed}.
Even though they and the proposed saddle-point framework lead to the same algorithm, the latter one is derived from a slightly different way based on a simple constrained convex optimization formulation, which are compatible with techniques in~\cite{qu2018exponential}.
In addition, we note that GTD5 proposed in this paper can be interpreted as GTD2 with a quadratic regularization term, which was also used in the distributed RLs in~\cite{cassano2019distributed,ding2019fast,wai2018multi}. Compared to them, GTD5 focuses on the single agent case, has different algorithmic structures, and uses diminishing weights on the regularization term in the comparative analysis. The so-called TD with Regularized Corrections (TDRC) was introduced in~\cite{ghiassian2020gradient}, which adds an additional term to TDC updates in~\cite{sutton2009fast} corresponding to $l_2$ regularization, and~\cite{patterson2022generalized} extends the ideas of GTDs to nonlinear function approximations.

\section{Preliminaries}
\subsection{Markov decision process}
A Markov decision process (MDP) is characterized by a quadruple ${\mathcal M}: =
({\cal S},{\mathcal A},P,r,\gamma)$, where ${\mathcal S}$ is a finite
state-space, $\cal A$ is a finite action
space, $P(s'|s,a)$ represents the (unknown)
state transition probability from state $s$ to $s'$ given action
$a$, $r:{\mathcal S}\times {\mathcal A}\times {\mathcal S}\to
{\mathbb R}$ is the reward
function, and $\gamma \in (0,1)$ is the discount factor. In particular, if action
$a$ is selected with the current state $s$, then the state
transits to $s'$ with probability $P(s'|s,a)$ and incurs a
reward $r(s,a,s')$. The stochastic policy is a map $\pi:{\mathcal S} \times
{\mathcal A}\to [0,1]$ representing the probability, $\pi(a|s)$, of selecting action $a$ at the current state $s$, $P^\pi$ denotes the transition matrix under policy $\pi$, and $d^{\pi}:{\mathcal S} \to {\mathbb R}$ denotes the stationary distribution of the state $s\in {\mathcal S}$ under $\pi$. We also define
$R^\pi(s)$ as the expected reward given the policy $\pi$ and the current state $s$. The infinite-horizon discounted value function with policy $\pi$ is $J^\pi(s):={\mathbb E} \left[ \left. \sum_{k = 0}^\infty {\gamma
^k r(s_k,a_k,s_{k+1})} \right|s_0=s \right]$, where ${\mathbb E}$ stands for the expectation taken with respect to the state-action trajectories under $\pi$. Given pre-selected basis (or feature) functions $\phi_1,\ldots,\phi_q:{\mathcal S}\to {\mathbb R}$, the matrix, $\Phi \in {\mathbb R}^{|{\mathcal S}| \times q}$, called the feature matrix, is defined as a matrix whose $s$-th row vector is $\phi(s):=\begin{bmatrix} \phi_1(s) &\cdots & \phi_q(s) \end{bmatrix}$. Throughout the paper, we assume that $\Phi \in {\mathbb R}^{|{\mathcal S}| \times q}$ is a full column rank matrix. The policy evaluation problem is the problem of estimating $J^{\pi}$ given a policy $\pi$.

\subsection{Basics of nonlinear system theory}
We will briefly review basic nonlinear system theory, which will play an important role in convergence analysis and stochastic approximation methods. Consider the nonlinear system
\begin{align}
\frac{d}{dt}x_t=f(x_t),\quad x_0\in {\mathbb R}^n,\quad t \geq 0,\label{eq:nonlinear-system}
\end{align}
where $x_t\in {\mathbb R}^n$ is the state, $t\geq 0$ is the time, $x_0\in {\mathbb R}^n$ is the initial state, and $f:{\mathbb R}^n \to {\mathbb R}^n$ is a nonlinear mapping. For simplicity, we assume that the solution to~\eqref{eq:nonlinear-system} exists and is unique. In fact, this holds true so long as the mapping $f$ is globally Lipschitz continuous.
\begin{lemma}[{\cite[Thm.~3.2]{khalil2002nonlinear}}]\label{lemma:existence}
Consider the nonlinear system~\eqref{eq:nonlinear-system}, and assume that $f$ is globally Lipschitz continuous, i.e., $\|f(x)-f(y)\|\le l \|x-y\|, \forall x,y \in {\mathbb R}^n$ for some $l>0$ and norm $\|\cdot\|$. Then, it admits a unique solution $x_t$ for all $t\geq 0$ and $x_0\in {\mathbb R}^n$.
\end{lemma}

The equilibrium point is an important concept in nonlinear system theory. In particular, a point, $x_\infty \in {\mathbb R}^n$, in the state-space is said to be an equilibrium point of~\eqref{eq:nonlinear-system} if whenever the state of the system starts at $x_\infty$, it will remain at $x_\infty$~\cite{khalil2002nonlinear}. For~\eqref{eq:nonlinear-system}, the equilibrium points are the real roots of the equation $f(x)=0$. The equilibrium point $x_\infty$ is said to be globally asymptotically stable if for any initial state $x_0 \in {\mathbb R}^n$, $x_t \to x_\infty$ as $t \to \infty$.

\subsection{ODE-based stochastic approximation}\label{sec:ODE-stochastic-approximation}
Due to its generality, the convergence analysis of many RL algorithms rely on the ODE (ordinary differential equation) approach~\cite{bhatnagar2012stochastic,kushner2003stochastic}. It analyzes convergence of general stochastic recursions by examining stability of the associated ODE model based on the fact that the stochastic recursions with diminishing step-sizes approximate the corresponding ODEs in the limit. One of the most popular approaches is based on Borkar and Meyn theorem~\cite{borkar2000ode}. We briefly review Borkar and Meyn's ODE approach, which analyzes convergence of the general stochastic recursions
\begin{align}
&\theta_{k+1}=\theta_k+\alpha_k (f(\theta_k)+\varepsilon_{k+1})\label{eq:general-stochastic-recursion}
\end{align}
where $f:{\mathbb R}^n \to {\mathbb R}^n$ is a nonlinear mapping. Basic technical assumptions are given below.
\begin{assumption}\label{assumption:1}
$\,$\begin{enumerate}
\item The mapping $f:{\mathbb R}^n  \to {\mathbb R}^n$ is
globally Lipschitz continuous, and there exists a function
$f_\infty:{\mathbb R}^n\to {\mathbb R}^n$ such that $\lim_{c\to \infty}\frac{f(c x)}{c}= f_\infty(x),\forall x \in {\mathbb R}^n$.

\item The origin in ${\mathbb R}^n$ is an asymptotically stable
equilibrium for the ODE $\dot \theta_t=f_\infty (\theta_t)$.

\item There exists a unique globally asymptotically stable equilibrium
$\theta_\infty \in {\mathbb R}^n$ for the ODE $\dot \theta_t=f(\theta_t)$, i.e., $\theta_t\to\theta_\infty$ as $t\to\infty$.

\item The sequence $\{\varepsilon_k,{\cal G}_k,k\ge 1\}$ with ${\cal G}_k=\sigma(\theta_i,\varepsilon_i,i\le k)$ is a Martingale difference sequence. In addition, there exists a constant $C_0<\infty $ such that for any initial $\theta_0\in {\mathbb R}^n$, we have ${\mathbb E}[\|\varepsilon_{k+1} \|_2^2 |{\cal G}_k]\le C_0(1+\|\theta_k\|_2^2),\forall k \ge 0$. Here, $\|\cdot \|_2$ denotes the standard Euclidean norm.

\item The step-size satisfies
\begin{align}
&\alpha_k>0,\quad \sum_{k=0}^\infty {\alpha_k}=\infty,\quad \sum_{k=0}^\infty{\alpha_k^2}<\infty.\label{eq:step-size-rule}
\end{align}
\end{enumerate}
\end{assumption}

\begin{lemma}[{\cite[Borkar and Meyn theorem]{borkar2000ode}}]\label{lemma:Borkar}
Suppose that~\cref{assumption:1} holds. Then, the following statements hold true:
\begin{enumerate}
\item For any initial $\theta_0\in
{\mathbb R}^n$, $\sup_{k\ge 0} \|\theta_k\|_2<\infty$
with probability one.

\item In addition, $\theta_k\to\theta_\infty$ as $k\to\infty$ with probability one.
\end{enumerate}
\end{lemma}
Borkar and Meyn theorem states that under~\cref{assumption:1}, the stochastic process $(\theta_k)_{k=0}^\infty$ generated by~\eqref{eq:general-stochastic-recursion} is bounded and converges to $\theta_\infty$ with probability one. It will be used to prove convergence of various algorithms throughout the paper.

\subsection{Saddle-point problem}\label{sec:saddle-point-problem}

In this subsection, we briefly review the saddle-point problem~\cite{nedic2009subgradient,qu2018exponential}. Consider a convex-concave function $L: {\mathbb R}^n \times {\mathbb R}^n \to {\mathbb R}$.
\begin{definition}[Saddle-point]
A saddle-point is defined as a pair $(\theta^*,\lambda^*) \in {\mathbb R}^n  \times {\mathbb R}^n$ that satisfies $L(\theta ^* ,\lambda ) \le L(\theta ^* ,\lambda ^* ) \le L(\theta ,\lambda ^* ),\forall (\theta ,\lambda ) \in {\mathbb R}^n  \times {\mathbb R}^n$.
\end{definition}
The saddle point problem is the problem of finding a saddle-point, which arises in a number of areas such as constrained optimization
duality, zero-sum games, and general equilibrium theory~\cite{nedic2009subgradient}.
Moreover, it is also known to be a solution of the min-max problem.
\begin{problem}[Min-max problem]\label{problem:saddle-point}
Find a pair $(\theta^*,\lambda^*) \in {\mathbb R}^n  \times {\mathbb R}^n$ which solves $\min _{\theta  \in {\mathbb R}^n } \max _{\lambda  \in {\mathbb R}^n } L(\theta ,\lambda ) = \max _{\lambda  \in {\mathbb R}^n } \min _{\theta  \in {\mathbb R}^n } L(\theta ,\lambda )$.
\end{problem}

Note that, $(\theta, \lambda)$ is a saddle-point if and only if the stationary point condition holds, i.e., $\nabla_\theta  L(\theta ^* ,\lambda ^* ) = \nabla _\lambda  L(\theta ^* ,\lambda ^* ) = 0$. The so-called primal-dual gradient method~\cite{nedic2009subgradient} is a popular method for solving~\cref{problem:saddle-point}:
\begin{align*}
\theta _{k + 1}  =& \theta _k  - \alpha _k \nabla _\theta  L(\theta_k ,\lambda_k ),\quad \lambda _{k + 1}  =& \lambda _k  + \alpha _k \nabla _\lambda  L(\theta_k ,\lambda_k ),
\end{align*}
where $(\alpha_k)_{k=0}^\infty$ is a step-size. This iteration will be called the discrete-time primal-dual gradient dynamics (PDGD) throughout the paper. Its continuous-time counterpart is
\begin{align*}
\dot \theta_t  =  - \nabla _\theta L(\theta_t ,\lambda_t),\quad \dot \lambda_t  = \nabla _\lambda  L(\theta_t ,\lambda_t ),
\end{align*}
and is called the continuous-time PDGD~\cite{qu2018exponential}. Both PDGDs converge to a saddle-point under some mild assumptions~\cite{nedic2009subgradient,qu2018exponential}.
If the gradients are not accessible, but only their stochastic approximations are available, then the stochastic counterpart is as follows:
\begin{align*}
\theta _{k + 1}  =& \theta _k  - \alpha _k \nabla _\theta  (L(\theta_k ,\lambda_k ) + v_k),\\
\lambda _{k + 1}  =& \lambda _k  + \alpha _k \nabla _\lambda  (L(\theta_k ,\lambda_k) + w_k),
\end{align*}
where $(v_k,w_k) \in {\mathbb R}^n \times {\mathbb R}^n$ is an i.i.d. noise with zero mean. In this paper, it will be called the stochastic PDGD. Stochastic PDGD also converges to a saddle-point in probabilistic senses~\cite{chen2016stochastic,wang2016online}. As an application, let us consider the constrained convex optimization problem.
\begin{problem}\label{problem:0}
Solve for $x \in {\mathbb R}^n$ the optimization
\[
\min_{x \in {\mathbb R}^n } f(x)\quad {\rm s.t.}\quad Ax = b,
\]
where $f$ is convex and continuously differentiable.
\end{problem}

The corresponding (convex-concave) Lagrangian function is defined as
\begin{align}
L(\theta ,\lambda ) = f(\theta ) + \lambda ^T (A\theta  - b),\label{eq:0}
\end{align}
where $\lambda \in {\mathbb R}^n$ is called the Lagrangian multiplier. Using the standard results in convex optimization theory, if~\cref{problem:0} satisfies the Slater's condition~\cite[Chap.~5]{Boyd2004}, then its solution can be found by solving the saddle-point problem in~\cref{problem:saddle-point}.

\section{Review of GTD algorithm}
In this section, we briefly review the gradient temporal difference (GTD) learning developed in~\cite{sutton2009convergent}, which tries to solve the policy evaluation problem. Roughly speaking, the goal of the policy evaluation is to find the weight vector $\theta$ such that $\Phi \theta$ approximates the true value function $J^{\pi}$. This is typically done by minimizing the so-called {\em mean-square Bellman error} loss function~\cite{sutton2009convergent}. The overall problem is summarized below.
\begin{problem}\label{problem:1}
Solve for $\theta \in {\mathbb R}^q$ the optimization
\begin{align*}
&\min_{\theta \in {\mathbb R}^q} {\rm MSBE}(\theta):=\frac{1}{2} \|
R^{\pi}+\gamma P^\pi\Phi \theta-\Phi \theta
\|_{D^{\beta}}^2,
\end{align*}
where $R^\pi \in {\mathbb R}^{|{\mathcal S}|}$ is a vector enumerating all $R^\pi(s), s\in {\mathcal S}$, $q$ is the number of feature functions, $D^{\beta}$ is a diagonal matrix with positive diagonal elements $d^{\beta}(s),s\in {\mathcal S}$, and $\|x\|_D:=\sqrt{x^T Dx}$ for any positive-definite $D$. Here, $d^{\beta}$ can be any state visit distribution under the behavior policy $\beta$ such that $d^{\beta}(s)>0,\forall s\in {\mathcal S}$.
\end{problem}
The GTD in~\cite{sutton2009fast} considers another objective function called the {\em mean-square projected Bellman error} loss function.
\begin{problem}\label{problem:2}
Solve for $\theta \in {\mathbb R}^q$ the optimization
\begin{align*}
&\min_{\theta\in {\mathbb R}^q} {\rm MSPBE}(\theta):= \frac{1}{2}\|
\Pi (R^{\pi} + \gamma P^{\pi} \Phi \theta)-\Phi \theta \|_{D^{\beta}}^2.
\end{align*}
\end{problem}
where $\Pi$ is the projection onto the range space of $\Phi$,
denoted by $R(\Phi)$: $\Pi(x):=\argmin_{x'\in R(\Phi)}
\|x-x'\|_{D^{\beta}}^2$. The projection can be performed by the matrix
multiplication: we write $\Pi(x):=\Pi x$, where $\Pi:=\Phi(\Phi^T
D^{\beta} \Phi)^{-1}\Phi^T D^{\beta}$. Note that minimizing the objective means minimizing the error of the projected Bellman equation $\Phi \theta  = \Pi (R^\pi   + \gamma P^\pi  \Phi \theta )$ with respect to $ \|\cdot\|_{D^\beta}$. Moreover, note that in the objective of~\cref{problem:2}, $d^{\beta}$ depends on the behavior policy, $\beta$, while $P^{\pi}$ and $R^{\pi}$ depend on the target policy, $\pi$, that we want to evaluate. This structure allows us to obtain an off-policy learning algorithm through the importance sampling~\cite{precup2001off} or sub-sampling techniques~\cite{sutton2009convergent}. Throughout the paper, we adopt the following standard assumption.
\begin{assumption}\label{assumption:2}
$\Phi ^T D^{\beta} (\gamma P^\pi   - I)\Phi $ is nonsingular, where $I$ denotes the identity matrix with an appropriate dimension.
\end{assumption}
Note that~\cref{assumption:2} is common in the literature, and is adopted in~\cite{sutton2009convergent,sutton2009fast,ghiassian2020gradient} for convergence of GTD algorithms. Some properties related to~\cref{problem:2} are summarized below for convenience and completeness.
\begin{lemma}\label{lemma:2}
The following statements hold true:
\begin{enumerate}
\item A solution of~\cref{problem:2} exists, and is unique.

\item The solution of~\cref{problem:2} is given by
\begin{align}
\theta ^*:=-(\Phi ^T D^{\beta} (\gamma P^\pi   - I)\Phi )^{ - 1} \Phi ^T D^{\beta} R^\pi.\label{eq:theta-star}
\end{align}
\end{enumerate}
\end{lemma}
\begin{proof}
For the first statement, the equation, $\Pi (R^\pi   + \gamma P^\pi  \Phi \theta ) - \Phi \theta  = 0$, can be equivalently written as $\Pi (R^\pi   + \gamma P^\pi  \Phi \theta ) - \Phi \theta  = 0 \Leftrightarrow \Phi (\Phi ^T D^\beta  \Phi )^{ - 1} \Phi ^T D^\beta  (R^\pi   + \gamma P^\pi  \Phi \theta ) - \Phi \theta  = 0\Leftrightarrow \Phi ^T D^\beta  (R^\pi   + \gamma P^\pi  \Phi \theta ) - \Phi ^T D^\beta  \Phi \theta  = 0 \Leftrightarrow \Phi ^T D^\beta  (\gamma P^\pi   - I)\Phi \theta  =  - \Phi ^T D^\beta  R^\pi$. Since $\Phi ^T D^\beta  (\gamma P^\pi   - I)\Phi$ is nonsingular by~\cref{assumption:2}, the last equation admits a unique solution. Moreover, the second statement is directly proved from the last equation. This completes the proof.
\end{proof}

Based on this objective function,~\cite{sutton2009fast} developed GTD2. The reader is referred to~\cite{sutton2009fast} for more details. After~\cite{sutton2009fast}, some different interpretations were developed based on saddle-point perspectives. Before proceeding, they are briefly presented in the following subsections.

\subsection{First approach: dual representatoin}
A saddle-point perspective of GTD2 was introduced in~\cite{macua2015distributed}.
The main idea is to convert~\cref{problem:2} into the equivalent quadratic constrained optimization problem
\begin{align*}
&\mathop {\min}\limits_{\theta,w  \in {\mathbb R}^q } \frac{1}{2}w^T (\Phi ^T D^\beta \Phi )^{ - 1} w\\
&{\rm s.t.}\quad w = \Phi ^T D^\beta (R^\pi   + \gamma P^\pi  \Phi \theta  - \Phi \theta )
\end{align*}
where $w \in {\mathbb R}^q$ is a newly introduced vector variable. Introducing the Lagrangian function $L(\theta ,w,\lambda ) = \frac{1}{2}w^T (\Phi ^T D^\beta \Phi )^{ - 1} w + \lambda ^T (\Phi ^T D^\beta (R^\pi   + \gamma P^\pi  \Phi \theta  - \Phi \theta)-w)$, where $\lambda\in {\mathbb R}^q$ is the Lagrangian multiplier, the dual problem~\cite{Boyd2004} is
\begin{align*}
&\min_{\lambda  \in \in {\mathbb R}^q} \frac{1}{2}\lambda ^T \Phi ^T D^\beta  \Phi \lambda  - \lambda ^T \Phi ^T D^\beta  R^\pi\\
&{\rm s.t.}\quad 0 = \lambda ^T \Phi ^T D^\beta (\gamma P^\pi  \Phi  - \Phi )
\end{align*}

The main reason to consider the dual problem instead of the primal problem is that the dual formulation removes the matrix inverse in the objective. Next, we can again construct the corresponding Lagrangian function for the dual problem as follow:
\[
L(\theta ,\lambda ) = \frac{1}{2}\lambda ^T \Phi ^T D^\beta  \Phi \lambda  - \lambda ^T \Phi ^T D^\beta  R^\pi   + \lambda ^T \Phi ^T D^\beta (\gamma P^\pi  \Phi  - \Phi )\theta
\]
where $\theta\in {\mathbb R}^q$ is the Lagrangian multiplier. Then, it turned out that GTD2 is identical to a stochastic PDGD for solving the saddle-point problem, $\min_{\lambda  \in {\mathbb R}^q } \max_{\theta  \in {\mathbb R}^q } L(\theta ,\lambda )$. For more details, the reader is referred to~\cite{macua2015distributed}.

\subsection{Second approach: Fenchel duality}
GTD2 can be also interpreted in a different direction using the Fenchel dual to~\cref{problem:2} as shown in~\cite{dai2018sbeed}. In particular, using the Fenchel duality, the conjugate form of ${\rm MSPBE}(\theta)$:
\begin{align*}
{\rm{MSPBE}}(\theta ) =& \frac{1}{2} \| {\Pi (R^\pi   + \gamma P^\pi  \Phi \theta ) - \Phi \theta } \|_{D^\beta  }^2\\
=& \frac{1}{2}\| {\Phi ^T D(R^\pi   + \gamma P^\pi  \Phi \theta  - \Phi \theta )} \|_{(\Phi ^T D^\beta \Phi )^{ - 1} }^2
\end{align*}
is given by
\begin{align*}
{\rm{MSPBE}}(\theta ) =& \max_{\lambda \in {\mathbb R}^q}  L(\theta ,\lambda )\\
:=& \lambda ^T \Phi ^T D^\beta (R^\pi   + \gamma P^\pi  \Phi \theta  - \Phi \theta ) - \frac{1}{2}\lambda ^T \Phi ^T D^\beta \Phi \lambda.
\end{align*}

Therefore,~\cref{problem:2} can be represented by the convex-concave saddle-point problem, $\min_{\theta  \in {\mathbb R}^q } {\rm MSPBE}(\theta ) = \min _{\theta  \in {\mathbb R}^q } \max _{\lambda  \in {\mathbb R}^q } L(\theta ,\lambda )$. Then, GTD2 is identical to a stochastic primal-dual algorithm for solving the above saddle-point problem. In the next section, we introduce an alternative saddle-point approach to derive GTD2 from a different angle.

\section{Third approach}
In this section, we introduce a slightly different approach to derive GTD2. To this end, let us consider the following constrained optimization problem.
\begin{problem}\label{problem:3}
Solve for $\theta  \in {\mathbb R}^q$ the optimization
\begin{align}
&\min_{\theta  \in {\mathbb R}^q } 0\quad {\rm s.t.}\quad 0 = \Phi ^T D^{\beta} (R^\pi   + \gamma P^\pi  \Phi \theta  - \Phi \theta ).\label{eq:2}
\end{align}
\end{problem}

Note that in~\cref{problem:3}, we introduce a null objective, $f\equiv 0$, to fit the problem into an optimization form. \cref{problem:3} can be seen as the projected Bellman equation $\Phi \theta  = \Pi (R^\pi   + \gamma P^\pi  \Phi \theta )$ in the form of an optimization problem. We can prove that the optimization admits a unique solution, which is identical to the solution of~\cref{problem:2}.
\begin{proposition}
A solution of~\cref{problem:3} exists and is unique given by $\theta^*$ defined in~\eqref{eq:theta-star}.
\end{proposition}
\begin{proof}
The equality constraint in~\eqref{eq:2} can be equivalently written as $0 = \Phi ^T D^\beta  (R^\pi   + \gamma P^\pi  \Phi \theta  - \Phi \theta ) \Leftrightarrow \Phi ^T D^\beta  (I - \gamma P^\pi  )\Phi \theta  = \Phi ^T D^\beta  R^\pi$, whose solution is identical to $\theta^*$ in~\eqref{eq:theta-star}. This completes the proof.
\end{proof}
To formulate~\cref{problem:3} into a min-max saddle-point problem, we introduce the corresponding Lagrangian function $L(\theta ,\lambda ) := \lambda ^T \Phi ^T D^{\beta} (R^\pi   + \gamma P^\pi  \Phi \theta  - \Phi \theta )$. Instead of directly deriving the corresponding dual problem, we introduce a regularization term to make it strongly concave in $\lambda$, and obtain the following modification:
\begin{align}
L(\theta ,\lambda ): = \lambda ^T \Phi ^T D^{\beta} (R^\pi   + \gamma P^\pi  \Phi \theta  - \Phi \theta ) - \frac{1}{2}\lambda ^T \Phi ^T D^{\beta} \Phi \lambda. \label{eq:Lagrangian1}
\end{align}

The corresponding saddle-point problem of~\eqref{eq:Lagrangian1} is then given as follows.
\begin{problem}\label{problem:4}
Solve for $(\theta ,\lambda )\in {\mathbb R}^q \times {\mathbb R}^q$ the min-max problem
\begin{align*}
\min_{\theta  \in {\mathbb R}^q } \max_{\lambda  \in {\mathbb R}^q } L(\theta ,\lambda ): =& \lambda ^T \Phi ^T D^{\beta} (R^\pi   + \gamma P^\pi  \Phi \theta  - \Phi \theta )\\
& - \frac{1}{2}\lambda ^T \Phi ^T D^{\beta} \Phi \lambda
\end{align*}
\end{problem}

Note that a quadratic penalty term (or regularization
term) has been added to the original Lagrangian function in~\eqref{eq:Lagrangian1}, which is not typical in terms of the standard
Lagrangian duality theory. In this sense,~\cref{problem:4} and~\cref{problem:3} are not equivalent. This additional term is introduced in order to derive GTD2 using the saddle-point viewpoints, and this process can give additional insights on GTD2.
Since~\cref{problem:4} is modified, a natural question is if the original equality constrained optimization in~\cref{problem:3} can be solved by addressing the saddle-point problem in~\cref{problem:4} for the regularized Lagrangian function~\eqref{eq:Lagrangian1}. We can conclude that the solutions of~\cref{problem:4} is indeed identical to those of~\cref{problem:2}.
\begin{proposition}\label{prop:2}
A solution of~\cref{problem:4} exists, is unique, and is given by $\theta = \theta^*$ and $\lambda = 0$.
\end{proposition}
\begin{proof}
Since $\Phi ^T D^\beta (\gamma P^\pi   - I)\Phi$ is nonsingular from~\cref{assumption:2}, $\nabla _\theta  L(\theta ,\lambda ) = (\gamma P^\pi  \Phi  - \Phi )^T D^{\beta} \Phi \lambda = 0$ implies $\lambda = 0$. On the other hand, $\nabla _\lambda  L(\theta ,\lambda ) = \Phi ^T D^{\beta} (R^\pi   + \gamma P^\pi  \Phi \theta  - \Phi \theta  - \Phi \lambda ) = 0$ with $\lambda = 0$ leads to the desired conclusion.
\end{proof}
Intuitively, the additional regularization term in~\eqref{eq:Lagrangian1} penalizes the Lagrangian multiplier $\lambda$ from being large, while this change does not affect the primal variable $\theta$. Now, let us turn our attention to its continuous-time PDGD~\cite{qu2018exponential}:
\begin{align}
\dot \theta _t  =&  - \nabla_\theta L(\theta_t ,\lambda _t ) =  - (\gamma P^\pi  \Phi  - \Phi )^T D^\beta  \Phi \lambda _t\nonumber\\
\dot \lambda _t  =& \nabla_\lambda L(\theta,\lambda _t ) = \Phi ^T D^\beta  (R^\pi   + \gamma P^\pi  \Phi \theta _t  - \Phi \theta _t  - \Phi \lambda _t )\label{eq:7}
\end{align}
Considering $s_k \sim d^{\beta}$, $a_k \sim \pi(\cdot|s_k)$, and $s_k'\sim P(\cdot |s_k,a_k)$, the corresponding stochastic PDGD can be obtained as follows:
\begin{align*}
\theta _{k + 1}  =& \theta _k  - \alpha _k (\gamma e_{s_k } e_{s_k '}^T \Phi  - \Phi )^T e_{s_k } e_{s_k }^T \Phi \lambda _k,\\
\lambda _{k + 1}  =& \lambda _k  + \alpha _k \Phi ^T e_{s_k } e_{s_k }^T (e_{s_k } r(s_k ,a_k ,s_k ')\\
& + \gamma e_{s_k } e_{s_k '}^T \Phi \theta _k  - \Phi \theta _k  - \Phi \lambda _k ).
\end{align*}
This recursion is identical to GTD2, which is summarized in~\cref{algo:GTD2} for completeness of presentations.
\begin{algorithm}[h]
\caption{GTD2}
\begin{algorithmic}[1]

\State Set the step-size $(\alpha_k)_{k=0}^\infty$ and initialize $(\theta _0,\lambda_0 )$.

\For{$k \in \{0,\ldots\}$}

\State Observe $s_k \sim d^{\beta}$, $a_k \sim \beta(\cdot|s_k)$, and $s_k'\sim P(\cdot |s_k,a_k)$, $r_k :=r(s_k,a_k,s_k')$.

\State Update parameters according to
\begin{align*}
&\theta_{k+1}=\theta_k + \alpha_k (\phi_k-\gamma \rho_k \phi_k')(\phi_k^T \lambda_k),\\
&\lambda_{k+1}=\lambda_k +\alpha_k (\delta_k -\phi_k^T \lambda_k)\phi_k,
\end{align*}
where $\phi_k:=\phi(s_k),\phi_{k}':=\phi(s_{k}')$, $\rho _k : = \frac{{\pi (a_k |s_k )}}{{\beta (a_k |s_k )}}$, and $\delta_k =\rho _k r_k +\gamma \rho _k (\phi_{k}')^T \theta_k -\phi_k^T \theta_k$.

\EndFor
\end{algorithmic}
\label{algo:GTD2}
\end{algorithm}

Note that in~\cref{algo:GTD2}, an importance sampling ratio, $\rho _k : = \pi (a_k |s_k ) /\beta (a_k |s_k )$, is introduced for off-policy learning~\cite{precup2001off}.

Although the convergence of GTD2 was given in~\cite{sutton2009fast}, we will provide another approach based on recent results in~\cite{qu2018exponential} in the next section.
\begin{remark}
A different algorithm can be obtained with the Lagrangian function $L(\theta ,\lambda ): = \lambda ^T \Phi ^T D^{\beta} (R^\pi   + \gamma P^\pi  \Phi \theta  - \Phi \theta ) - \frac{1}{2}\lambda ^T \lambda$, which may have different convergence properties. In general, the corresponding algorithm performs better with smaller step-sizes, while in general, GTD2 converges faster.
\end{remark}

\section{Convergence of GTD2}
In this section, we will provide an alternative approach to the convergence of GTD2 based on the recent results in~\cite{qu2018exponential} in combination with the constrained optimization perspective of GTD2 in the previous section. Before proceeding, some results of~\cite{qu2018exponential} are briefly summarized.
\begin{lemma}[Thm.~1, \cite{qu2018exponential}]\label{lemma:1}
Consider the equality constrained optimization in~\cref{problem:0}, and suppose that $f$ is twice differentiable, $\mu$-strongly convex, and $l$-smooth, i..e, for all $x,y \in {\mathbb R}^n$, $\mu \left\| {x - y} \right\|_2^2  \le (\nabla f(x) - \nabla f(y))^T (x-y) \le l\left\| {x - y} \right\|_2^2$. Moreover, suppose that $A$ is full row rank. Consider the corresponding Lagrangian function~\eqref{eq:0}. Then, the corresponding saddle-point $(\theta^*,\lambda^*)$ is unique, and the corresponding continuous-time PDGD,
\begin{align*}
\dot x_t =&  - \nabla _\theta  L(\theta_t ,\lambda_t ) =  - \nabla _\theta  f(\theta_t ) - A^T \lambda_t\\
\dot \lambda_t =& \nabla _\lambda  L(\theta_t ,\lambda_t ) = A\theta_t  - b,
\end{align*}
exponentially converges to $(\theta^*,\lambda^*)$.
\end{lemma}

In the sequel, we will apply~\cref{lemma:1} to prove the convergence of GTD2, especially, for the global asymptotic stability of its ODE model.
The main difficulty in applying~\cref{lemma:1} to~\cref{problem:3} is that~\cref{problem:3} has a null objective, $f\equiv 0$, which does not satisfy the strong convexity assumption of the objective function in~\cref{lemma:1}. To resolve this problem, we will consider the dual problem of~\cref{problem:3} instead of its original form.
\begin{problem}[Dual problem]\label{problem:7}
Solve for $\lambda\in {\mathbb R}^q$ the optimization
\begin{align*}
&\max_{\lambda\in {\mathbb R}^q }  \lambda ^T \Phi ^T D^{\beta} R^\pi   - \frac{1}{2}\lambda ^T \Phi ^T D^{\beta} \Phi \lambda\\
&{\rm s.t.}\quad \Phi ^T (\gamma P^\pi   - I)^T D^{\beta} \Phi \lambda  = 0
\end{align*}
\end{problem}
\begin{proposition}
\cref{problem:7} is the dual problem of~\cref{problem:3}.
\end{proposition}
\begin{proof}
Consider the Lagrangian function~\eqref{eq:Lagrangian1}, and the corresponding min-max problem in~\cref{problem:4}. The Lagrangian function can be written by $L(\theta ,\lambda ) = \lambda ^T \Phi ^T D^{\beta} R^\pi - \frac{1}{2}\lambda ^T \Phi ^T D^{\beta} \Phi \lambda  + (\lambda ^T \Phi ^T D^{\beta} \gamma P^\pi  \Phi  - \lambda ^T \Phi ^T D^{\beta} \Phi )\theta$. If we fix $\lambda$, then the problem $\min_{\theta  \in {\mathbb R}^q } L(\theta ,\lambda )$ has a finite optimal
value, when $\lambda ^T \Phi ^T D^{\beta} \gamma P^\pi  \Phi  - \lambda ^T \Phi ^T D^{\beta} \Phi  = 0$. Therefore, the dual problem $\max_{\lambda  \in {\mathbb R}^q } \min_{\theta\in {\mathbb R}^q } L(\theta ,\lambda )$ is~\cref{problem:7}. This completes the proof.
\end{proof}

Now,~\cref{problem:7} (maximization problem) can be equivalently written as the minimization problem
\begin{align}
&\min_{\lambda \in {\mathbb R}^q}  \frac{1}{2}\lambda ^T \Phi ^T D^{\beta} \Phi \lambda - \lambda ^T \Phi ^T D^{\beta}  R^\pi\label{eq:modified-dual}\\
&{\rm{s}}{\rm{.t}}{\rm{.}}\quad \lambda ^T \Phi ^T D^{\beta} \Phi  - \lambda ^T \Phi ^T D^{\beta} \gamma P^\pi  \Phi  = 0\nonumber
\end{align}
The corresponding Lagrangian function is
\begin{align*}
&L(\theta ,\lambda ) = \frac{1}{2}\lambda ^T \Phi ^T D^{\beta} \Phi \lambda\\
& - \lambda ^T \Phi ^T D^{\beta} R^\pi   + \lambda ^T \Phi ^T D^{\beta} \Phi \theta  - \lambda ^T \Phi ^T D^{\beta} \gamma P^\pi  \Phi \theta,
\end{align*}
and the corresponding continuous-time PDGD is
\begin{align*}
\dot \lambda_t  =&  - \nabla _\lambda  L(\theta_t ,\lambda_t ) =  - \Phi ^T D^{\beta} (\gamma P^\pi  \Phi \theta _t - \Phi \theta_t  - R^\pi   + \Phi \lambda_t),\\
\dot \theta_t  =& \nabla _\theta  L(\theta_t ,\lambda_t ) = (\Phi  - \gamma P^\pi  \Phi )^T D^{\beta} \Phi \lambda_t.
\end{align*}

We can easily check that the PDGD is identical to that of~\cref{problem:3}, given in~\eqref{eq:7}.
Therefore,~\cref{lemma:1} can be applied to~\cref{problem:3} in place of~\cref{problem:7}.
\begin{proposition}\label{prop:1}
Consider the trajectory $(\theta_t, \lambda_t)$ of the PDGD in~\eqref{eq:7}. Then, $(\theta_t, \lambda_t)\to (\theta^*, 0)$ exponentially as $t\to \infty$.
\end{proposition}
\begin{proof}
Note that~\eqref{eq:modified-dual} has a strongly convex, smooth, and twice differentiable objective function. Moreover, $\Phi ^T (I - \gamma P^\pi  )^T D^\beta \Phi$ is nonsingular by~\cref{assumption:2}, and hence is full row rank. The other assumptions are also met. Therefore, we can apply~\cref{lemma:1} to obtain the desired conclusion.
\end{proof}

Now, we can easily apply Borkar and Meyn theorem with~\cref{prop:1} to complete the proof. Details of the remaining parts can be found in~\cite{sutton2009fast}.
\begin{lemma}[Thm.~1, \cite{sutton2009fast}]\label{thm:convergence1}
Consider~\cref{algo:GTD2}, and assume that the step-size satisfy~\eqref{eq:step-size-rule}. Then, $\theta_k \to \theta^*$ and $\lambda_k \to 0$ as $k \to \infty$ with probability one.
\end{lemma}
\begin{remark}
\cref{lemma:1} provides an exponential convergence, while~\cref{thm:convergence1} provides an asymptotic convergence. The convergence result in this paper relies on the standard ODE methods (Borkar and Meyn theorem), which do not provide convergence rates in
general even if the corresponding O.D.E model's solution converges exponentially fast. Convergence rate analysis can be a potential future topic.
\end{remark}

In this section, we proposed a saddle-point interpretation of GTD2 from a slightly different perspective, and presented a different analysis for the global stability of the corresponding ODE model. Starting from this new perspective, we will provide two new versions of GTD in the next section.

\section{GTD3}

In this section, we propose a new optimization formulation for the policy evaluation problem,~\cref{problem:2}. Based on this form, we will derive another version of GTD, called GTD3 in this paper.
\begin{problem}\label{problem:5}
Solve for $\theta  \in {\mathbb R}^q$ the optimization
\begin{align*}
&\min_{\theta  \in {\mathbb R}^q } \frac{1}{2}\theta ^T \Phi ^T D^{\beta} \Phi \theta\\
&{\rm s.t.}\quad 0 = \Phi ^T D^{\beta} (R^\pi   + \gamma P^\pi  \Phi \theta  - \Phi \theta )
\end{align*}
\end{problem}

Compared to~\cref{problem:3} for GTD2, the main difference is that it has a quadratic objective instead of the null objective. A natural question is if this optimization admits the identical solution to~\cref{problem:3}. The answer is indeed positive.
\begin{proposition}
A solution of~\cref{problem:5} exists, and is unique given by $\theta = \theta^*$.
\end{proposition}
\begin{proof}
It is clear from~\cref{prop:2} that the inequality constraint has a unique feasible point $\theta = \theta^*$. Therefore, the optimal solution is also uniquely determined by $\theta^*$. This completes the proof.
\end{proof}

To derive a saddle-point formulation again, let us consider the Lagrangian function for~\cref{problem:5}
\begin{align}
L(\theta ,\lambda ) = \frac{1}{2}\theta ^T \Phi ^T D^{\beta} \Phi \theta  + \lambda ^T \Phi ^T D^{\beta} (R^\pi   + \gamma P^\pi  \Phi \theta  - \Phi \theta )\label{eq:9}
\end{align}
Note that it is concave in $\lambda$ and strongly convex in $\theta$. Compared to the Lagrangian function of GTD2 in~\eqref{eq:Lagrangian1}, the regularization term, $- \frac{1}{2}\lambda ^T \Phi ^T D^{\beta} \Phi \lambda$, which is strongly concave in $\lambda$, is replaced with the regularization term, $\frac{1}{2}\theta ^T \Phi ^T D^{\beta} \Phi \theta$, which is strongly convex in $\theta$.
The corresponding min-max saddle-point formulation is given as follows.
\begin{problem}\label{problem:6}
Solve for $(\theta ,\lambda )\in {\mathbb R}^q \times {\mathbb R}^q$ the optimization
\begin{align*}
\min_{\theta  \in {\mathbb R}^q } \max_{\lambda  \in {\mathbb R}^q } L(\theta ,\lambda ): =& \frac{1}{2}\theta ^T \Phi ^T D^{\beta} \Phi \theta \\
& + \lambda ^T \Phi ^T D^{\beta} (R^\pi   + \gamma P^\pi  \Phi \theta  - \Phi \theta ).
\end{align*}
\end{problem}

Again, we can conclude that the solutions of~\cref{problem:6} is also identical to those of~\cref{problem:2}.
\begin{proposition}\label{prop:6}
\cref{problem:6} admits a unique solution given by $\theta  = \theta^*$ and $\lambda  = \lambda^*$, where $\theta^*$ is given in~\cref{lemma:2}, and
\begin{align}
\lambda^*:=& (\Phi ^T (\gamma P^\pi   - I)^T D^{\beta} \Phi )^{ - 1} \Phi ^T D^{\beta} \Phi\nonumber\\
&\times (\Phi ^T D^{\beta} (\gamma P^\pi   - I)\Phi )^{ - 1} \Phi ^T D^{\beta} R^\pi.\label{eq:8}
\end{align}
\end{proposition}
The results in~\cref{prop:6} can be easily obtained by solving the stationary point condition for~\cref{problem:6}, i.e., $\nabla _\theta  L(\theta ,\lambda ) = 0$ and $\nabla _\lambda  L(\theta ,\lambda ) = 0$. Therefore, the detailed proof is omitted here.
Similar to the previous section, the continuous-time PDGD of~\cref{problem:5} (or equivalently,~\cref{problem:6}) is
\begin{align}
\dot \theta_t =&  - \Phi ^T D^\beta \Phi \theta_t  - (\gamma P^\pi  \Phi  - \Phi )^T D^\beta \Phi \lambda_t, \nonumber\\
\dot \lambda_t  =& \Phi ^T D^\beta (R^\pi   + \gamma P^\pi  \Phi \theta_t  - \Phi \theta_t ),\label{eq:4}
\end{align}
and its discrete-time counterpart (by Euler discretization) is
\begin{align}
\theta _{k + 1}  =& \theta _k  - \alpha _k (\Phi ^T D^\beta \Phi \theta _k  + (\gamma P^\pi  \Phi  - \Phi )^T D^\beta \Phi \lambda _k ),\nonumber\\
\lambda _{k + 1}  =& \lambda _k  + \alpha _k \Phi ^T D^\beta (R^\pi   + \gamma P^\pi  \Phi \theta _k  - \Phi \theta _k ).\label{eq:3}
\end{align}

With the samples $s_k \sim d^{\beta}$, $a_k \sim \pi(\cdot|s_k)$, and $s_k'\sim P(\cdot |s_k,a_k)$, a stochastic approximation~\cite{wang2016online,chen2016stochastic} of the discrete-time counterpart is given as
\begin{align*}
\theta _{k + 1}  =& \theta _k  - \alpha _k (\Phi ^T e_{s_k } e_{s_k }^T \Phi \theta _k  + \Phi ^T (\gamma e_{s_k } e_{s_{k + 1} }^T  - I)^T \Phi \lambda _k ),\nonumber\\
\lambda _{k + 1}  =& \lambda _k  + \alpha _k \Phi ^T (e_{s_k } e_{s_k }^T r_k  + \gamma e_{s_k } e_{s_{k + 1} }^T \Phi \theta _k  - e_{s_k } e_{s_k }^T \Phi \theta _k ).
\end{align*}
The proposed algorithm is summarized in~\cref{algo:GTD3}, which is  equivalent to the updates in the above recursion with the importance sampling~\cite{precup2001off}.
\begin{algorithm}[h]
\caption{GTD3}
\begin{algorithmic}[1]

\State Set the step-size $(\alpha_k)_{k=0}^\infty$, and initialize $(\theta _0,\lambda_0 )$.

\For{$k \in \{0,\ldots\}$}

\State Observe $s_k \sim d^{\beta}$, $a_k \sim \beta(\cdot|s_k)$, and $s_k'\sim P(\cdot | s_k,a_k)$, $r_k :=r(s_k,a_k,s_k')$.

\State Update parameters according to
\begin{align*}
&\theta _{k + 1}  = \theta _k  + \alpha _k [(\phi_k  - \gamma \rho _k \phi_{k}')(\phi_k ^T \lambda _k ) - \phi_k (\phi_k ^T \theta_k )],\\
&\lambda _{k + 1}  = \lambda _k  + \alpha _k \delta _k \phi_k,
\end{align*}
where $\phi_k:=\phi(s_k),\phi_{k}':=\phi(s_{k}')$, $\rho_k : = \frac{{\pi (a_k |s_k )}}{{\beta (a_k |s_k )}}$, and $\delta_k =\rho _k r_k +\gamma \rho _k (\phi_{k}')^T \theta_k -\phi_k^T \theta_k$.

\EndFor
\end{algorithmic}
\label{algo:GTD3}
\end{algorithm}

Note that~\cref{algo:GTD3} is different from GTD2, linear TD with
gradient correction (TDC)~\cite{sutton2009fast}, and the original GTD in~\cite{sutton2009convergent}.
Moreover, the optimization problem corresponding to~\cref{algo:GTD3} (\cref{problem:5}) already has the required structures in~\cref{lemma:1}. Therefore,~\cref{lemma:1} can be directly applied to prove the global stability of the corresponding PDGD in~\eqref{eq:4}. Note also that the PDGD in~\eqref{eq:4} is also the ODE model of~\cref{algo:GTD3}.
\begin{proposition}\label{prop:3}
Consider the trajectory $(\theta_t, \lambda_t)$ of the PDGD in~\eqref{eq:4}. Then, $(\theta_t, \lambda_t)\to (\theta^*, \lambda^*)$ as $t\to \infty$.
\end{proposition}
\begin{proof}
\cref{problem:5} has a strongly convex, smooth, and twice differentiable objective function. Moreover, $\Phi ^T (I - \gamma P^\pi  )^T D^{\beta} \Phi$ is nonsingular by~\cref{assumption:2}, and hence is full row rank. The other assumptions are also met. Therefore, the PDGD of~\cref{problem:5}, given in~\eqref{eq:4}, is globally asymptotically stable, and converges to its unique equilibrium point $(\theta^*, \lambda^*)$, where $\lambda^*$ is defined in~\eqref{eq:8}. This completes the proof.
\end{proof}

Based on~\cref{prop:3}, convergence of~\cref{algo:GTD3} can be proved using the Bokar and Mayn theorem.
\begin{theorem}\label{thm:convergence2}
Consider~\cref{algo:GTD3}, and assume that the step-size satisfy~\eqref{eq:step-size-rule}.
Then, $\theta_k \to \theta^*$ as $k\to \infty$ with probability one.
\end{theorem}
The proof of~\cref{thm:convergence2} can be found in Appendix.

\begin{remark}
A different algorithm can be obtained with the following Lagrangian function $L(\theta ,\lambda ) = \frac{1}{2}\theta ^T \theta  + \lambda ^T \Phi ^T D^{\beta} (R^\pi   + \gamma P^\pi  \Phi \theta  - \Phi \theta )$, which has different convergence properties. In general, the corresponding algorithm performs better with smaller step-sizes, while in general, GTD3 converges faster.
\end{remark}

In this section, we proposed a new version of GTD based on a new saddle-point formulation in~\cref{problem:6}. In the next section, we propose another version of GTD based on the Lagrangian function which has more symmetric form.

\section{GTD4 \& 5}

Let us recall the Lagrangian functions~\eqref{eq:Lagrangian1} for GTD2 and~\eqref{eq:9} for GTD3.
The Lagrangian function in~\eqref{eq:9} replaces $\frac{1}{2}\theta ^T \Phi ^T D^{\beta} \Phi \theta$ in~\eqref{eq:Lagrangian1} with $-\frac{1}{2}\lambda ^T \Phi ^T D^{\beta} \Phi \lambda$.
In this section, we investigate the function
\begin{align}
L(\theta ,\lambda ) =& \frac{\sigma}{2}\theta ^T \Phi ^T D^{\beta} \Phi \theta + \lambda ^T \Phi ^T D^{\beta} (R^\pi   + \gamma P^\pi  \Phi \theta  - \Phi \theta )\nonumber\\
& - \frac{1 }{2}\lambda ^T \Phi ^T D^{\beta} \Phi \lambda, \label{eq:10}
\end{align}
where $\sigma  \geq 0$ is a weight on the first regularization term, i.e., a design parameter. Note that with $\sigma =0$,~\eqref{eq:10} is reduced to~\eqref{eq:Lagrangian1}. The function includes both $\frac{1}{2}\lambda ^T \Phi ^T D^{\beta} \Phi \lambda$ and $\frac{1}{2}\theta ^T \Phi ^T D^{\beta} \Phi \theta$, and hence, more symmetric than~\eqref{eq:Lagrangian1} and~\eqref{eq:9}.
Moreover, it is strongly convex in $\theta$ and strongly concave in $\lambda$. The corresponding min-max saddle-point problem is summarized below for convenience.
\begin{problem}\label{problem:8}
Solve for $(\theta ,\lambda )\in {\mathbb R}^q \times {\mathbb R}^q$ the min-max problem
\begin{align*}
&\min _{\theta\in {\mathbb R}^q} \max_{\lambda\in {\mathbb R}^q}  L(\theta ,\lambda )\\
:=& \frac{\sigma}{2}\theta ^T \Phi ^T D^\beta\Phi \theta  + \lambda ^T \Phi ^T D^\beta (R^\pi   + \gamma P^\pi  \Phi \theta  - \Phi \theta )\\
& - \frac{1}{2}\lambda ^T \Phi ^T D^\beta \Phi \lambda.
\end{align*}
\end{problem}

Solving the stationary point conditions, $\nabla _\theta  L(\theta ,\lambda )=0$ and $\nabla _\lambda  L(\theta ,\lambda ) = 0$, we obtain the following solution of~\cref{problem:8} for the $\theta$-coordinate: $\hat \theta _\sigma : =  - (\Phi ^T D^{\beta} (\gamma P^\pi   - I)\Phi  + \sigma E)^{ - 1} \Phi ^T D^{\beta} R^\pi$, where $E: = \Phi ^T D^{\beta}\Phi (\Phi ^T (\gamma P^\pi   - I)^T D^{\beta} \Phi )^{ - 1} \Phi ^T D^{\beta} \Phi$.

From the result, it turns out that the solution of~\cref{problem:8} for the $\theta$-coordinate is not exactly identical to $\theta^*$, but includes a bias term $\sigma E$. However, since $\hat \theta _\sigma   \to \theta ^*$ as $\sigma \to 0$, one can control the degree of the error by adjusting $\sigma$. Moreover, larger $\sigma$ can more stabilize the final algorithm, and speed up its convergence because it improves the degree of the strong concavity of~\eqref{eq:10}. Therefore, there exists a trade-off between stability and bias by choosing $\sigma$. Besides, the solution is formally stated in the following proposition for convenience.
\begin{proposition}\label{prop:4}
The unique saddle-point of~\cref{problem:8} for $\theta$-coordinate is given by $\hat \theta _\sigma : =  - (\Phi ^T D^{\beta} (\gamma P^\pi   - I)\Phi  + \sigma E)^{ - 1} \Phi ^T D^{\beta} R^\pi$, where $E: = \Phi ^T D^{\beta} \Phi (\Phi ^T (\gamma P^\pi   - I)^T D^{\beta} \Phi )^{ - 1} \Phi ^T D^{\beta} \Phi$.
\end{proposition}
The proof of~\cref{prop:4} is a direct calculation, and so omitted here for brevity.
Similar to the previous section, the continuous-time PDGD is
\begin{align}
\dot \theta _t  =&  - \nabla _\theta  L(\theta_t ,\lambda_t ) =  - \Phi ^T D\Phi \theta _t  - (\gamma P^\pi  \Phi  - \Phi )^T D^\beta \Phi \lambda _t\nonumber\\
\dot \lambda _t  =& \nabla _\lambda  L(\theta_t ,\lambda_t ) = \Phi ^T D^\beta (R^\pi   + \gamma P^\pi  \Phi \theta _t  - \Phi \theta _t ) - \sigma \Phi ^T D^\beta \Phi \lambda _t\label{eq:11}
\end{align}

In the following theorem, we first establish the global asymptotic stability of the PDGD in~\eqref{eq:11}.
\begin{proposition}\label{prop:5}
Consider the trajectory $(\theta_t, \lambda_t)$ of the PDGD in~\eqref{eq:11}, and let $(\hat \theta_\sigma,\hat \lambda_\sigma)$ be the corresponding unique saddle-point. Then, $(\theta_t, \lambda_t)\to (\hat \theta_\sigma,\hat \lambda_\sigma)$ as $t\to \infty$.
\end{proposition}
\begin{proof}
The PDGD in~\eqref{eq:11} can be rewritten by{\small
\begin{align*}
&\frac{d}{{dt}}\left[ {\begin{array}{*{20}c}
   {\theta_t - \hat \theta_\sigma }  \\
   {\lambda_t - \hat \lambda_\sigma }  \\
\end{array}} \right] = \left[ {\begin{array}{*{20}c}
   { - \nabla _\theta  L(\theta_t ,\lambda_t )}  \\
   {\nabla _\lambda  L(\theta_t ,\lambda_t )}  \\
\end{array}} \right]\\
& = \left[ {\begin{array}{*{20}c}
   { - \Phi ^T D\Phi } & { - (\gamma P^\pi  \Phi  - \Phi )^T D^\beta \Phi }  \\
   {\Phi ^T D^\beta (\gamma P^\pi  \Phi  - \Phi )\theta } & { - \sigma \Phi ^T D^\beta \Phi }  \\
\end{array}} \right]\left[ {\begin{array}{*{20}c}
   {\theta_t - \hat \theta_\sigma }  \\
   {\lambda_t - \hat \lambda_\sigma }  \\
\end{array}} \right].
\end{align*}}

The proof for $\sigma = 0$ is given in the proof of~\cref{prop:3}.
Therefore, we only consider the case $\sigma > 0$ here. Let us consider the positive definite Lyapunov function candidate
$V(\theta  - \hat \theta _\sigma  ,\lambda  - \hat \lambda _\sigma  ) := (\theta  - \hat \theta _\sigma  )^T (\theta  - \hat \theta _\sigma  ) + (\lambda  - \hat \lambda_\sigma  )^T (\lambda  - \hat \lambda _\sigma  )$. It's time-derivative along the trajectory is{\small
\begin{align*}
&\frac{d}{{dt}}V(\theta_t - \hat \theta _\sigma  ,\lambda_t - \hat \lambda _\sigma  )\\
=& \left[ {\begin{array}{*{20}c}
   {\theta_t  - \hat \theta _\sigma  }  \\
   {\lambda_t - \hat \lambda _\sigma  }  \\
\end{array}} \right]^T \left[ {\begin{array}{*{20}c}
   { - 2\Phi ^T D^\beta \Phi } & 0  \\
   0 & { - 2\sigma \Phi ^T D^\beta \Phi }  \\
\end{array}} \right]\left[ {\begin{array}{*{20}c}
   {\theta _t - \hat \theta _\sigma  }  \\
   {\lambda_t  - \hat \lambda _\sigma  }  \\
\end{array}} \right]\\
 <& 0,\quad \forall \theta_t - \hat \theta _\sigma   \ne 0,\lambda_t - \hat \lambda _\sigma   \ne 0.
\end{align*}}
Therefore, by the Lyapunov theorem~\cite{khalil2002nonlinear}, the system is globally asymptotically stable. This completes the proof.
\end{proof}

With the samples $s_k \sim d^{\beta}$, $a_k \sim \pi(\cdot|s_k)$, and $s_k'\sim P(\cdot |s_k,a_k)$, a stochastic PDGD corresponding to~\eqref{eq:11} is given as
\begin{align*}
\theta _{k + 1} =& \theta _k  - \alpha _k (\Phi ^T e_{s_k } e_{s_k }^T \Phi \theta _k  + (\gamma e_{s_k } e_{s_k '}^T \Phi  - \Phi )^T e_{s_k } e_{s_k }^T \Phi \lambda _k ),\\
\lambda _{k + 1}  =& \lambda _k  + \alpha _k (\Phi ^T e_{s_k } e_{s_k }^T (e_{s_k } r_k  + \gamma e_{s_k } e_{s_k '}^T \Phi \theta _k  - \Phi \theta _k )\\
& - \sigma \Phi ^T e_{s_k } e_{s_k }^T \Phi \lambda _k ),
\end{align*}
which is the second proposed algorithm, called GTD4 in this paper. The overall algorithm with the importance sampling (for off-policy learning) is summarized in~\cref{algo:GTD4}.
\begin{algorithm}[h]
\caption{GTD4}
\begin{algorithmic}[1]

\State Set the step-size $(\alpha_k)_{k=0}^\infty$, and initialize $(\theta _0,\lambda_0 )$.

\For{$k \in \{0,\ldots\}$}

\State Observe $s_k \sim d^{\beta}$, $a_k \sim \beta(\cdot|s_k)$, and $s_k'\sim P(\cdot |s_k,a_k)$, $r_k :=r(s_k,a_k,s_k')$.

\State Update parameters according to
\begin{align*}
&\theta _{k + 1}  = \theta _k  + \alpha _k [(\phi_k  - \gamma \rho _k \phi_{k}')(\phi_k^T \lambda _k ) - \phi_k (\phi_k ^T \theta_k )]\\
&\lambda _{k + 1}  = \lambda _k  + \alpha _k (\delta _k  - \sigma \phi_k^T \lambda_k )\phi_k
\end{align*}
where $\phi_k:=\phi(s_k),\phi_{k}':=\phi(s_{k}')$, $\rho _k : = \frac{{\pi (a_k |s_k )}}{{\beta (a_k |s_k )}}$, and $\delta_k =\rho _k r_k +\gamma \rho _k (\phi_{k}')^T \theta_k -\phi_k^T \theta_k$.

\EndFor
\end{algorithmic}
\label{algo:GTD4}
\end{algorithm}

With~\cref{prop:5}, one can easily prove the convergence of~\cref{algo:GTD4} using the ODE method in~\cref{lemma:Borkar}.
\begin{theorem}\label{thm:convergence3}
Consider~\cref{algo:GTD4}, and assume that the step-size satisfy~\eqref{eq:step-size-rule}.
Then, $\theta_k \to \hat \theta_\sigma$ as $k \to \infty$ with probability one.
\end{theorem}
\begin{proof}
By~\cref{prop:5}, the PDGD in~\eqref{eq:11} is globally asymptotically stable. The remaining parts of the proof are almost identical to those of~\cref{thm:convergence2}, and hence, are omitted here for brevity.
\end{proof}
\begin{remark}
The main difference of~\cref{algo:GTD4} (GTD4) and~\cref{algo:GTD2} (GTD2) lies in the existence of the additional regularization term in their Lagrangian functions in~\eqref{eq:10} and~\eqref{eq:Lagrangian1}.
By adding the regularization term, the saddle-point problem becomes strongly convex-concave, which may potentially accelerate the convergence. Intuitively, the additional regularization terms make the slopes of the Lagrangian functions more steep in both ascent and descent directions, and hence, one can expect that the stochastic gradient ascent and descent methods may converge faster to the stationary points.
\end{remark}

Finally, a modification of~\eqref{eq:10} leads to GTD5
\begin{align*}
L(\theta ,\lambda ) =& \frac{\sigma}{2}\theta ^T\theta + \lambda ^T \Phi ^T D^{\beta} (R^\pi   + \gamma P^\pi  \Phi \theta  - \Phi \theta )\\
& - \frac{1 }{2}\lambda ^T \Phi ^T D^{\beta} \Phi \lambda,
\end{align*}
where $ \frac{\sigma}{2}\theta ^T \Phi ^T D^{\beta} \Phi \theta $ in~\eqref{eq:10} is replaced with $\frac{\sigma}{2}\theta ^T\theta$ in the above equation. The algorithm is summarized in~\cref{algo:GTD5}.
\begin{algorithm}[h]
\caption{GTD5}
\begin{algorithmic}[1]

\State Set the step-size $(\alpha_k)_{k=0}^\infty$, and initialize $(\theta _0,\lambda_0 )$.

\For{$k \in \{0,\ldots\}$}

\State Observe $s_k \sim d^{\beta}$, $a_k \sim \beta(\cdot|s_k)$, and $s_k'\sim P(\cdot |s_k,a_k)$, $r_k :=r(s_k,a_k,s_k')$.

\State Update parameters according to
\begin{align*}
&\theta _{k + 1}  = \theta _k  + \alpha _k [(\phi_k  - \gamma \rho _k \phi_{k}')(\phi_k^T \lambda _k ) - \phi_k (\phi_k ^T \theta_k )]\\
&\lambda _{k + 1}  = \lambda _k  + \alpha _k (\delta _k\phi_k - \sigma \lambda_k)
\end{align*}
where $\phi_k:=\phi(s_k),\phi_{k}':=\phi(s_{k}')$, $\rho _k : = \frac{{\pi (a_k |s_k )}}{{\beta (a_k |s_k )}}$, and $\delta_k =\rho _k r_k +\gamma \rho _k (\phi_{k}')^T \theta_k -\phi_k^T \theta_k$.

\EndFor
\end{algorithmic}
\label{algo:GTD5}
\end{algorithm}

Since $\sigma>0$ leads to biases in solutions, a reasonable heuristic approach is to diminish $\sigma$, i.e., $\sigma \to 0$ as $k \to \infty$. A comparative analysis of several GTDs is given in the next section.

\section{Comparative analysis}
We randomly generated MDPs with $100$ states and $10$ actions, and behavior policies, and set $\gamma = 0.9$. The reward function was generated such that $r(s,a,s')$ is uniformly distributed over $[-1,1]$. Then, to make it sparse, elements with $|r(s,a,s')|\leq 0.2$ were set to be zero.
Similarly, $10$ feature functions were generated such that each element uniformly distributed over $[-1,1]$, and $\Phi$ is full column rank.
\begin{figure}[t]
\centering\subfigure[Step-size $\alpha _k  = 5/(k + 5)$]{\includegraphics[width=7cm,height=4.5cm]{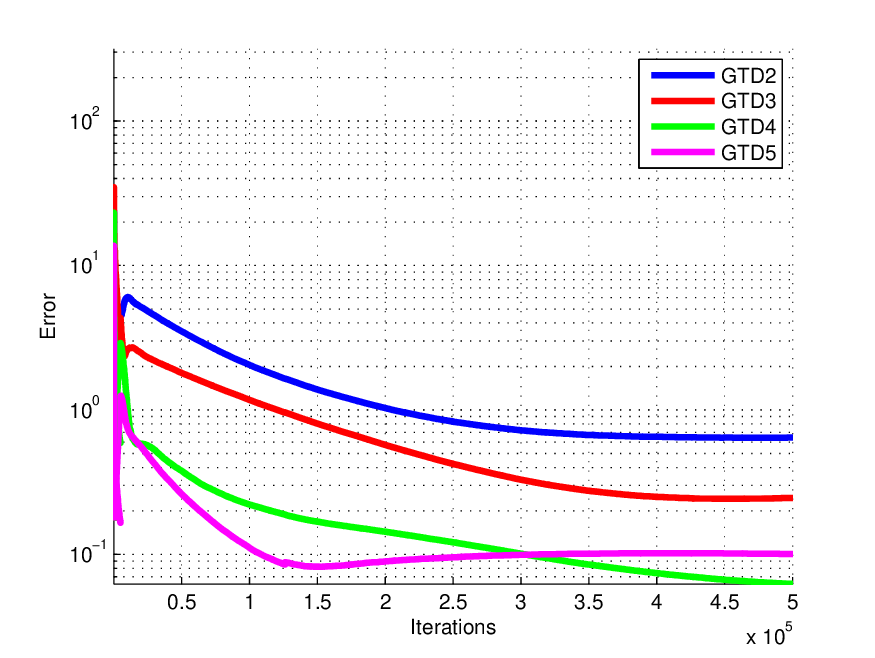}}
\centering\subfigure[Step-size $\alpha _k  = 10/(k + 10)$]{\includegraphics[width=7cm,height=4.5cm]{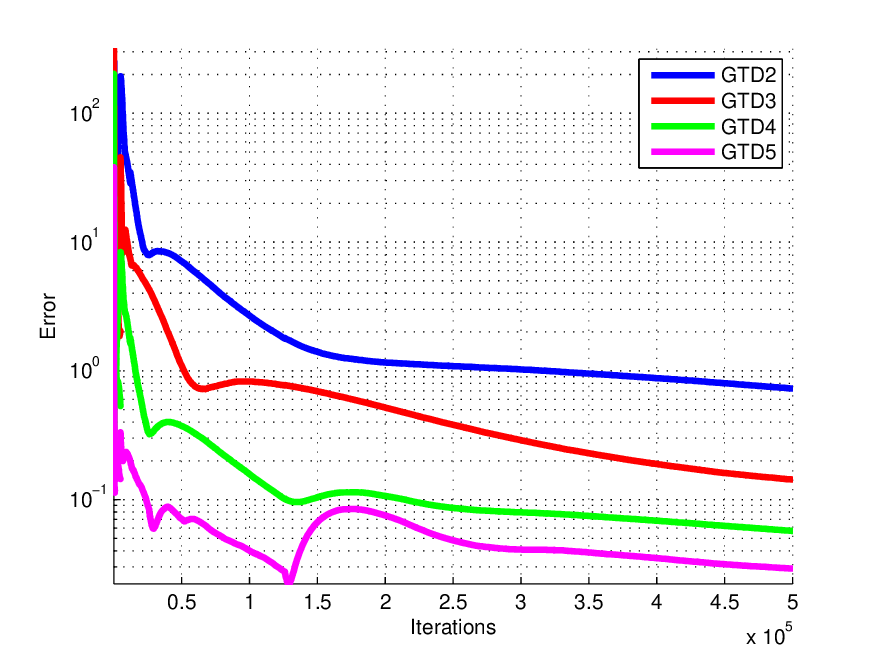}}
\caption{First instance: (a) Evolution of error, $\left\| \theta_k  - \theta^* \right\|$, for step-size $\alpha _k  = 5/(k + 5)$; (b) Evolution of error, $\left\| \theta_k  - \theta^* \right\|$, for step-size $\alpha _k  = 10/(k + 10)$. The figure illustrates error evolutions for GTD2 (blue), GTD3 (red), GTD4 (green), GTD5 (magenta) in a logarithmic scale. For GTD4 and GTD5, we used a diminishing $\sigma$: $\sigma _k  = 100/(k + 100)$. }\label{fig:1}
\end{figure}
\begin{figure}[t]
\centering\subfigure[Step-size $\alpha _k  = 5/(k + 5)$]{\includegraphics[width=7cm,height=4.5cm]{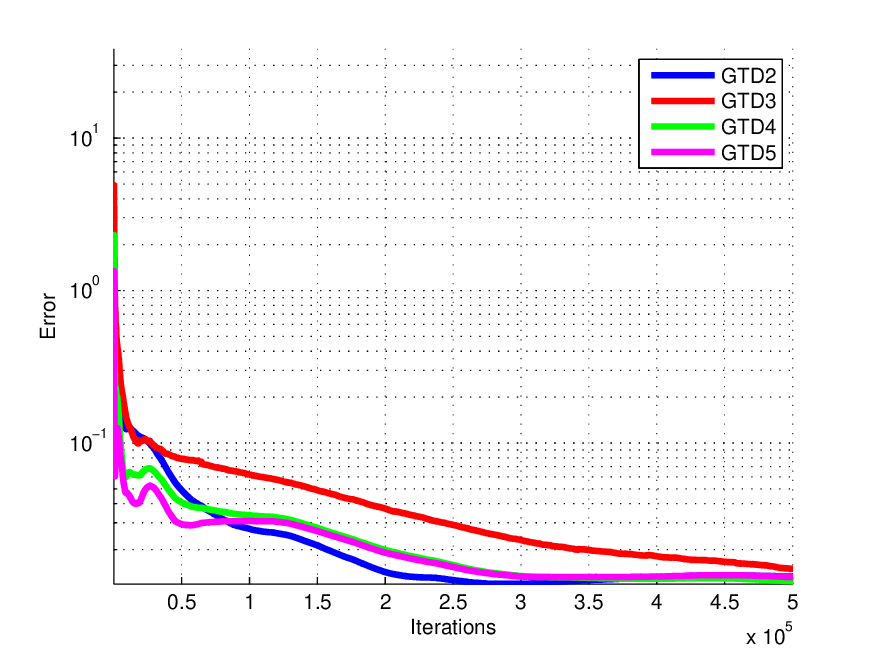}}
\centering\subfigure[Step-size $\alpha _k  = 10/(k + 10)$]{\includegraphics[width=7cm,height=4.5cm]{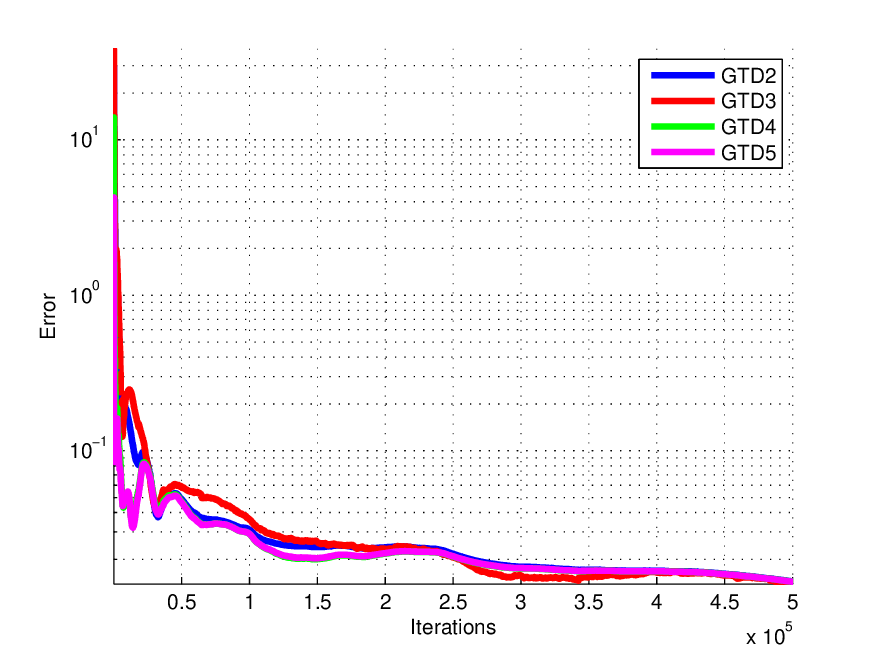}}
\caption{Second instance: (a) Evolution of error, $\left\| \theta_k  - \theta^* \right\|$, for step-size $\alpha _k  = 5/(k + 5)$; (b) Evolution of error, $\left\| \theta_k  - \theta^* \right\|$, for step-size $\alpha _k  = 10/(k + 10)$. The figure illustrates error evolutions for GTD2 (blue), GTD3 (red), GTD4 (green), GTD5 (magenta) in a logarithmic scale. For GTD4 and GTD5, we used a diminishing $\sigma$: $\sigma _k  = 100/(k + 100)$.}\label{fig:2}
\end{figure}
\cref{fig:1} shows error evolutions for different GTDs, GTD2 (blue line), GTD3 (red line), GTD4 (green line), GTD5 (magenta line), in a logarithmic scale. \cref{fig:1}{(a)} depicts results for step-size $\alpha _k  = 5/(k + 5)$, and \cref{fig:1}{(b)}for step-size $\alpha _k  = 10/(k + 10)$. For GTDs~4 and~5, we used a diminishing $\sigma$: $\sigma _k  = 100/(k + 100)$. The step-sizes were selected such that all the algorithms perform reasonably well. The results show an instance where GTD4 and GTD5 overcome GTD2, and GTD3. GTD3 converges slightly faster than GTD2 in this example.
\begin{remark}
The diminishing weight $\sigma_k$ has been selected from trial and errors. Intuitively, if the weight $\sigma_k$ diminishes too fast, then the algorithm quickly becomes identical to the standard GTD2 and GTD3. Therefore, the convergence speeds also become similar to GTD2 and GTD3. For the algorithms to be effective, the weight, $\sigma_k$, should not diminish too fast. On the other hand, if $\sigma_k$ diminishes too slowly, then the bias induced by the regularization vanishes too slowly during the learning. Therefore, there exists a trade-off between the convergence speed and bias, which leads to some tuning issues.
\end{remark}

\cref{fig:2} provides another instance where GTD2 performs slightly better than or equal to the other approaches. From our experiences, GTD4 and GTD5 outperform the other two approaches more frequently. For a fair and more comprehensive analysis, we ranked the four approaches based on the performance index $\sum_{k = 0}^{\tau} {\| \theta _k  - \theta ^* \|_2/1000}$ for 5000 randomly generated MDPs, where $N$ is the total number of iterations set to be $\tau= 250000$ in this example. In addition to the random generation scheme used in the previous two MDP instances, we also randomly select the number of states and number of actions uniformly distributed in $\{3,4,\ldots, 100 \}$ and $\{2,3,\ldots, 30 \}$, respectively. The number of feature functions is chosen such that it is around $1/10$ of the state size. Rankings of the different GTDs for 5000 MDP instances are summarized in~\cref{fig:3}, where each bar implies the number of MDP instances where the corresponding ranking is achieved by each method in terms of the performance index. We used the step-size $\alpha _k  = 5/(k + 5)$ and diminishing weight $\sigma _k  = 100/(k + 100)$ for this experiment. GTD5 takes the first place most frequently (3316 times over 5000 trials), and GTD4 takes the second-best places most frequently (3016 times over 5000 trials). GTD2 and GTD3 are comparable to each other.
The results suggest that GTD5 and GTD4 outperform the other approaches in most cases.
\begin{figure}[t]
\centering\includegraphics[width=7cm,height=4.5cm]{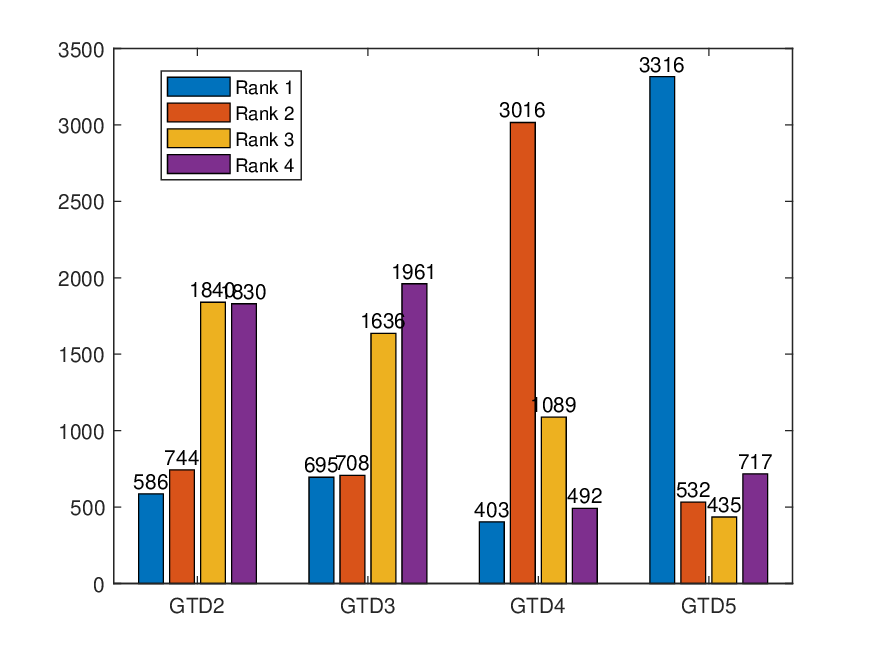}
\caption{Rankings of GTDs for 5000 MDP instances. }\label{fig:3}
\end{figure}

\section{Conclusion}
In this paper, we proposed variants of GTDs based on convex-concave saddle-point interpretations of GTDs, which allow new stability analysis based on recent results~\cite{qu2018exponential} on stability of PDGD. Performance of the GTDs was evaluated through numerical experiments, which suggest that GTD4 and GTD5 overcome the other methods for randomly generated 5000 MDPs. Therefore, we can conclude that the use of regularization terms with diminishing weights can potentially improve the convergence speed empirically. Besides, the convergence rate analysis is an important next step beyond the asymptotic convergence as the former analysis gives insights on how fast the iterates approach to the solution. Moreover, more combinations of regularization methods can lead to more versions of GTD, which have different properties and performances. These topics can be potential future works.

\bibliographystyle{IEEEtran}
\bibliography{reference}

\begin{thebibliography}{10}
\providecommand{\url}[1]{#1}
\csname url@samestyle\endcsname
\providecommand{\newblock}{\relax}
\providecommand{\bibinfo}[2]{#2}
\providecommand{\BIBentrySTDinterwordspacing}{\spaceskip=0pt\relax}
\providecommand{\BIBentryALTinterwordstretchfactor}{4}
\providecommand{\BIBentryALTinterwordspacing}{\spaceskip=\fontdimen2\font plus
\BIBentryALTinterwordstretchfactor\fontdimen3\font minus
  \fontdimen4\font\relax}
\providecommand{\BIBforeignlanguage}[2]{{%
\expandafter\ifx\csname l@#1\endcsname\relax
\typeout{** WARNING: IEEEtran.bst: No hyphenation pattern has been}%
\typeout{** loaded for the language `#1'. Using the pattern for}%
\typeout{** the default language instead.}%
\else
\language=\csname l@#1\endcsname
\fi
#2}}
\providecommand{\BIBdecl}{\relax}
\BIBdecl

\bibitem{sutton1988learning}
R.~S. Sutton, ``Learning to predict by the methods of temporal differences,''
  \emph{Machine learning}, vol.~3, no.~1, pp. 9--44, 1988.

\bibitem{sutton1998reinforcement}
R.~S. Sutton and A.~G. Barto, \emph{Reinforcement learning: {A}n
  introduction}.\hskip 1em plus 0.5em minus 0.4em\relax MIT Press, 1998.

\bibitem{sutton2009convergent}
R.~S. Sutton, H.~R. Maei, and C.~Szepesv{\'a}ri, ``A convergent {$O(n)$}
  temporal-difference algorithm for off-policy learning with linear function
  approximation,'' in \emph{Advances in neural information processing systems},
  2009, pp. 1609--1616.

\bibitem{sutton2009fast}
R.~S. Sutton, H.~R. Maei, D.~Precup, S.~Bhatnagar, D.~Silver,
  C.~Szepesv{\'a}ri, and E.~Wiewiora, ``Fast gradient-descent methods for
  temporal-difference learning with linear function approximation,'' in
  \emph{Proceedings of the 26th Annual International Conference on Machine
  Learning}, 2009, pp. 993--1000.

\bibitem{macua2015distributed}
S.~V. Macua, J.~Chen, S.~Zazo, and A.~H. Sayed, ``Distributed policy evaluation
  under multiple behavior strategies,'' \emph{IEEE Transactions on Automatic
  Control}, vol.~60, no.~5, pp. 1260--1274, 2015.

\bibitem{dai2018sbeed}
B.~Dai, A.~Shaw, L.~Li, L.~Xiao, N.~He, Z.~Liu, J.~Chen, and L.~Song,
  ``{SBEED}: {C}onvergent reinforcement learning with nonlinear function
  approximation,'' in \emph{International Conference on Machine Learning},
  2018, pp. 1125--1134.

\bibitem{Boyd2004}
S.~Boyd and L.~Vandenberghe, \emph{Convex Optimization}.\hskip 1em plus 0.5em
  minus 0.4em\relax Cambridge University Press, 2004.

\bibitem{cassano2019distributed}
L.~Cassano, K.~Yuan, and A.~H. Sayed, ``Distributed value-function learning
  with linear convergence rates,'' in \emph{2019 18th European Control
  Conference (ECC)}, 2019, pp. 505--511.

\bibitem{ding2019fast}
D.~Ding, X.~Wei, Z.~Yang, Z.~Wang, and M.~R. Jovanovic, ``Fast multi-agent
  temporal-difference learning via homotopy stochastic primal-dual method,'' in
  \emph{Optimization Foundations for Reinforcement Learning Workshop, 33rd
  Conference on Neural Information Processing Systems}, 2019.

\bibitem{wai2018multi}
H.-T. Wai, Z.~Yang, Z.~Wang, and M.~Hong, ``Multi-agent reinforcement learning
  via double averaging primal-dual optimization,'' \emph{Advances in Neural
  Information Processing Systems}, vol.~31, 2018.

\bibitem{lee2022distributed}
D.~Lee, J.~Hu \emph{et~al.}, ``Distributed off-policy temporal difference
  learning using primal-dual method,'' \emph{IEEE Access}, 2022.

\bibitem{kushner2003stochastic}
H.~Kushner and G.~G. Yin, \emph{Stochastic approximation and recursive
  algorithms and applications}.\hskip 1em plus 0.5em minus 0.4em\relax Springer
  Science \& Business Media, 2003, vol.~35.

\bibitem{qu2018exponential}
G.~Qu and N.~Li, ``On the exponential stability of primal-dual gradient
  dynamics,'' \emph{IEEE Control Systems Letters}, vol.~3, no.~1, pp. 43--48,
  2018.

\bibitem{ghiassian2020gradient}
S.~Ghiassian, A.~Patterson, S.~Garg, D.~Gupta, A.~White, and M.~White,
  ``Gradient temporal-difference learning with regularized corrections,'' in
  \emph{International Conference on Machine Learning}, 2020, pp. 3524--3534.

\bibitem{patterson2022generalized}
A.~Patterson, A.~White, and M.~White, ``A generalized projected {B}ellman error
  for off-policy value estimation in reinforcement learning,'' \emph{Journal of
  Machine Learning Research}, vol.~23, no. 145, pp. 1--61, 2022.

\bibitem{khalil2002nonlinear}
H.~K. Khalil, ``Nonlinear systems,'' \emph{Upper Saddle River}, 2002.

\bibitem{bhatnagar2012stochastic}
S.~Bhatnagar, H.~Prasad, and L.~Prashanth, \emph{Stochastic recursive
  algorithms for optimization: simultaneous perturbation methods}.\hskip 1em
  plus 0.5em minus 0.4em\relax Springer, 2012, vol. 434.

\bibitem{borkar2000ode}
V.~S. Borkar and S.~P. Meyn, ``The {ODE} method for convergence of stochastic
  approximation and reinforcement learning,'' \emph{SIAM Journal on Control and
  Optimization}, vol.~38, no.~2, pp. 447--469, 2000.

\bibitem{nedic2009subgradient}
A.~Nedi{\'c} and A.~Ozdaglar, ``Subgradient methods for saddle-point
  problems,'' \emph{Journal of optimization theory and applications}, vol. 142,
  no.~1, pp. 205--228, 2009.

\bibitem{chen2016stochastic}
Y.~Chen and M.~Wang, ``Stochastic primal-dual methods and sample complexity of
  reinforcement learning,'' \emph{arXiv preprint arXiv:1612.02516}, 2016.

\bibitem{wang2016online}
M.~Wang and Y.~Chen, ``An online primal-dual method for discounted markov
  decision processes,'' in \emph{55th IEEE Conference on Decision and Control
  (CDC)}, 2016, pp. 4516--4521.

\bibitem{precup2001off}
D.~Precup, R.~S. Sutton, and S.~Dasgupta, ``Off-policy temporal-difference
  learning with function approximation,'' in \emph{ICML}, 2001, pp. 417--424.

\end{thebibliography}

\appendices

\section{Proof of~\cref{thm:convergence2}}\label{app:2}
The algorithm in~\eqref{eq:3} can be written as
\begin{align}
\left[ {\begin{array}{*{20}c}
   {\theta _{k + 1} }  \\
   {\lambda _{k + 1} }  \\
\end{array}} \right] = \left[ {\begin{array}{*{20}c}
   {\theta _k }  \\
   {\lambda _k }  \\
\end{array}} \right] + \alpha _k \left( {f\left( {\left[ {\begin{array}{*{20}c}
   {\theta _k }  \\
   {\lambda _k }  \\
\end{array}} \right]} \right) + \varepsilon _{k + 1} } \right),\label{eq:5}
\end{align}
where
\begin{align*}
f\left( {\left[ {\begin{array}{*{20}c}
   \theta   \\
   \lambda   \\
\end{array}} \right]} \right): =& \left[ {\begin{array}{*{20}c}
   {\Phi ^T D^\beta \Phi } & {\Phi ^T (\gamma P^\pi   - I)^T D^\beta \Phi }  \\
   {\Phi ^T D^\beta (\gamma P^\pi   - I)\Phi } & 0  \\
\end{array}} \right]\left[ {\begin{array}{*{20}c}
   \theta   \\
   \lambda   \\
\end{array}} \right]\\
& + \left[ {\begin{array}{*{20}c}
   0  \\
   {\Phi^T D^\beta R^\pi  }  \\
\end{array}} \right],
\end{align*}
and
\begin{align*}
\varepsilon _{k + 1}  =& \left[ {\begin{array}{*{20}c}
   {\Phi ^T e_{s_k } e_{s_k }^T \Phi \theta _k  + \Phi ^T (\gamma e_{s_k } e_{s_{k + 1} }^T  - I)^T \Phi \lambda _k }  \\
   {\Phi ^T (e_{s_k } e_{s_k }^T r_k  + \gamma e_{s_k } e_{s_{k + 1} }^T \Phi \theta _k  - e_{s_k } e_{s_k }^T \Phi \theta _k )}  \\
\end{array}} \right]\\
& - f\left( {\left[ {\begin{array}{*{20}c}
   {\theta _k }  \\
   {\lambda _k }  \\
\end{array}} \right]} \right)
\end{align*}

The proof is completed by examining all the statements in~\cref{assumption:1}:
\begin{enumerate}
\item To prove the first statement of~\cref{assumption:1}, we have
\begin{align*}
&\mathop {\lim }\limits_{c \to \infty } f\left( {c\left[ {\begin{array}{*{20}c}
   \theta   \\
   \lambda   \\
\end{array}} \right]} \right)/c\\
=& f_\infty  \left( {\left[ {\begin{array}{*{20}c}
   \theta   \\
   \lambda   \\
\end{array}} \right]} \right)\\
 =& \left[ {\begin{array}{*{20}c}
   {\Phi ^T D^\beta \Phi } & {\Phi ^T (\gamma P^\pi   - I)^T D^\beta \Phi }  \\
   {\Phi ^T D^\beta (\gamma P^\pi   - I)\Phi } & 0  \\
\end{array}} \right]\left[ {\begin{array}{*{20}c}
   \theta   \\
   \lambda   \\
\end{array}} \right].
\end{align*}

Moreover, since $f$ is affine in its argument, it is globally Lipschitz continuous.

\item The second statement of~\cref{assumption:1}: The PDGD of~\cref{problem:5} can be written as
\begin{align*}
&\frac{d}{{dt}}\left[ {\begin{array}{*{20}c}
   {\theta_t - \theta _\infty }  \\
   {\lambda_t - \lambda _\infty }  \\
\end{array}} \right]\\
=& \left[ {\begin{array}{*{20}c}
   { - \Phi ^T D^\beta \Phi } & { - (\gamma P^\pi  \Phi  - \Phi )^T D^\beta \Phi }  \\
   \Phi ^T D^\beta (\gamma P^\pi  \Phi  - \Phi ) & 0  \\
\end{array}} \right]\left[ {\begin{array}{*{20}c}
   {\theta_t - \theta _\infty }  \\
   {\lambda_t - \lambda _\infty }  \\
\end{array}} \right],
\end{align*}
whose origin is the globally asymptotically stable equilibrium point. Now, one can observe that it is identical to the ODE
\[
\frac{d}{{dt}}\left[ {\begin{array}{*{20}c}
   {\theta  - \theta _\infty }  \\
   {\lambda  - \theta _\infty }  \\
\end{array}} \right] = f_\infty  \left( {\left[ {\begin{array}{*{20}c}
   {\theta_t - \theta _\infty }  \\
   {\lambda_t - \theta _\infty }  \\
\end{array}} \right]} \right).
\]
Therefore, its origin is the globally asymptotically stable equilibrium point.

\item The third statement of~\cref{assumption:1}: The ODE, $\frac{d}{{dt}}\left[ {\begin{array}{*{20}c}
   \theta_t \\
   \lambda_t \\
\end{array}} \right] = f\left( {\left[ {\begin{array}{*{20}c}
   \theta_t \\
   \lambda_t \\
\end{array}} \right]} \right)$, is identical to the PDGD of~\cref{problem:5}. Therefore, it admits a unique globally asymptotically stable equilibrium point by~\cref{lemma:1}.

\item Next, we prove the remaining parts. Recall that the GTD update

Define the history  ${\cal G}_k : = (\varepsilon _k ,\varepsilon _{k - 1} , \ldots ,\varepsilon _1 ,\theta _k ,\theta _{k - 1} , \ldots ,\theta _0 ,\lambda _k ,\lambda _{k - 1} , \ldots ,\lambda _0 )$, and
the process $(M_k)_{k=0}^\infty$ with $M_k:=\sum_{i=1}^k {\varepsilon_i}$. Then, we can prove that $(M_k)_{k=0}^\infty$ is Martingale. To do so, we first prove ${\mathbb E}[\varepsilon_{k+1}|{\cal G}_k]=0$ by
\begin{align*}
&{\mathbb E}[\varepsilon _{k + 1} |{\cal G}_k ]\\
=& {\mathbb E}\left[ {\left. {\left[ {\begin{array}{*{20}c}
   {\Phi ^T e_{s_k } e_{s_k }^T \Phi \theta _k  + \Phi ^T (\gamma e_{s_k } e_{s_{k + 1} }^T  - I)^T \Phi \lambda _k }  \\
   {\Phi ^T (e_{s_k } e_{s_k }^T r_k  + \gamma e_{s_k } e_{s_{k + 1} }^T \Phi \theta _k  - e_{s_k } e_{s_k }^T \Phi \theta _k )}  \\
\end{array}} \right]} \right|{\cal G}_k } \right]\\
& - {\mathbb E}\left[ {\left. {f\left( {\left[ {\begin{array}{*{20}c}
   {\theta _k }  \\
   {\lambda _k }  \\
\end{array}} \right]} \right)} \right|{\cal G}_k } \right]\\
=&{\mathbb E} \left[ {\left. {f\left( {\left[ {\begin{array}{*{20}c}
   {\theta _k }  \\
   {\lambda _k }  \\
\end{array}} \right]} \right)} \right|{\cal G}_k } \right] - {\mathbb E}\left[ {\left. {f\left( {\left[ {\begin{array}{*{20}c}
   {\theta _k }  \\
   {\lambda _k }  \\
\end{array}} \right]} \right)} \right|{\cal G}_k } \right] = 0,
\end{align*}
where the second equality is due to the i.i.d. assumption of samples. Using this identity, we have
\begin{align*}
{\mathbb E}[M_{k+1}|{\cal G}_k]=& {\mathbb E}\left[ \left. \sum_{i=1}^{k+1}{\varepsilon_i} \right|{\cal G}_k\right]={\mathbb E}[\varepsilon_{k+1}|{\cal G}_k]+{\mathbb E}\left[ \left. \sum_{i=1}^k {\varepsilon_i} \right|{\cal G}_k \right]\\
=&{\mathbb E}\left[\left.\sum_{i=1}^k{\varepsilon_i} \right|{\cal G}_k \right]=\sum_{i=1}^k {\varepsilon_i}=M_k.
\end{align*}
Therefore, $(M_k)_{k=0}^\infty$ is a Martingale sequence, and $\varepsilon_{k+1} = M_{k+1}-M_k$ is a Martingale difference. Moreover, it can be easily proved that the second statement of the fourth condition of~\cref{assumption:1} is satisfied by algebraic calculations. Therefore, the fourth condition is met.
\end{enumerate}

\end{document}